\documentclass[letterpaper, 10 pt, conference]{ieeeconf}  

\IEEEoverridecommandlockouts                              

\usepackage{cite}
\usepackage{amsmath,amssymb,amsfonts}
\usepackage{graphicx}
\usepackage{textcomp}
\usepackage{xcolor}
\usepackage{epstopdf}
\usepackage{epsfig} 
\usepackage{times} 
\usepackage{tabularx}
\usepackage{dcolumn}
\usepackage{color} 
\usepackage{algorithm}
\usepackage[algo2e,linesnumbered,vlined,ruled]{algorithm2e} 
\usepackage{balance}
\usepackage{accents}
\usepackage{resizegather}
\usepackage{theorem} 
\usepackage{mathtools}
\usepackage{multirow}
\usepackage{amsmath}

\usepackage{bm}
\usepackage{wrapfig}
\usepackage{graphicx}
\usepackage{subcaption}
\usepackage{caption}
\usepackage{mdframed}
\usepackage{algpseudocode}
\usepackage{hyperref}
\usepackage[ruled]{algorithm2e}
\usepackage{booktabs}

\usepackage{caption}
\newtheorem{definition}{\bf{Definition}}

\newtheorem{problem}{\bf{Problem}}

\newtheorem{theorem}{\bf{Theorem}}

\newtheorem{proposition}[theorem]{\bf{Proposition}}

\newcommand{\LTLUNTIL}{\ensuremath{U}}
\newcommand{\LTLEVENTUALLY}{\ensuremath{ F}}
\newcommand{\LTLALWAYS}{\ensuremath{ G }}

\newcommand{\notltl}{\neg}
\newcommand{\andltl}{\wedge}
\newcommand{\orltl}{\vee}

\SetKw{KwAnd}{and}

\title{\LARGE \bf
Specification-Aware Distribution Shaping for Robotics Foundation Models

\author{Sad{\i}k Bera Y\"{u}ksel and Derya Aksaray}
\thanks{S.B. Y\"{u}ksel is a PhD student in the Department of Electrical and Computer Engineering at Northeastern University.}%
\thanks{D. Aksaray is an Assistant Professor in the Department of Electrical and Computer Engineering at Northeastern University.}
}

\begin{document}

\maketitle
\thispagestyle{empty}
\pagestyle{empty}

\begin{abstract}
Robotics foundation models have demonstrated strong capabilities in executing natural language instructions across diverse tasks and environments. However, they remain largely data-driven and lack formal guarantees on constraint satisfaction during deployment. In practice, robots often need to comply with operational constraints involving rich spatio-temporal requirements such as time-bounded goal visits, sequential objectives, and persistent safety conditions. In this work, we propose a specification-aware action distribution optimization framework that enforces a broad class of Signal Temporal Logic (STL) constraints during execution of a pretrained robotics foundation model without modifying its parameters. At each decision step, the method computes a \emph{minimally modified} action distribution that satisfies a hard STL feasibility constraint by reasoning over the remaining horizon using forward dynamics propagation. We validate the proposed framework in simulation using state-of-the-art robotics foundation models across multiple environments and complex specifications.
\end{abstract}

\section{INTRODUCTION}
Recent advances in robotics foundation models have significantly expanded the capabilities of embodied agents. These models are pretrained on large-scale multi-modal data to learn unified vision-language-action representations that allow robots to map high-level commands directly to low-level control (e.g., \cite{doi:10.1177/02783649241281508, hu2024generalpurposerobotsfoundationmodels}). As a result, they demonstrate strong performance across navigation and manipulation benchmarks and achieve high task completion rates in diverse settings (e.g., \cite{SPOC, Poliformer, FlaRe, pmlr-v270-kim25c}). For example, a robot can interpret instructions such as “go to the kitchen and pick up a mug” or “find a pen” and autonomously execute the task with high success rates. Furthermore, these models exhibit strong zero-shot generalization, which enables them to operate in previously unseen layouts and handle novel objects and task variations without additional retraining.

Despite these impressive capabilities, robotics foundation models remain largely data-driven and lack formal guarantees on their behavior during deployment. In safety-critical settings, robots must satisfy strict operational constraints, such as avoiding restricted regions and respecting geofencing limits. Some approaches incorporate safety considerations into foundation models through training or inference-time modifications. For example, \cite{NEURIPS2025_safeVLA} considers safety constraints that penalize behaviors such as collisions with objects or unsafe interactions during manipulation, while \cite{pmlr-v305-sermanet25a} introduces semantic safety rules that discourage harmful actions, such as manipulating fragile or hazardous objects. However, real-world deployments often require enforcing richer spatio-temporal requirements that govern how robot behaviors evolve over time, including but not limited to sequential task specifications, time-bounded objectives, and persistent safety conditions that must hold throughout execution. Existing foundation models do not provide a formal mechanism to ensure compliance with such constraints at runtime.

In the literature, a prominent way of expressing complex operational requirements is via temporal logic (TL), which is a formal specification language to describe and reason about complex spatio-temporal behaviors \cite{baier2008principles}. Temporal logics can compactly and rigorously express constraints such as “always avoid unsafe regions,” “eventually reach a goal within a time interval,” “complete tasks in a prescribed sequence”, ``if an event is detected, immediately go to the safe region". In this work, we focus on Signal Temporal Logic (STL) \cite{maler2004monitoring}, which provides quantitative semantics over real-valued signals with explicit state and time parameters. While TL has been successfully applied in robotics for safety-constrained planning and control, its integration with large transformer-based foundation models remains limited.

Several works have incorporated TL into foundation-model-based robot decision-making. SELP \cite{SELP}, \cite{10611447}, and LIMP \cite{LIMP} use LTL to formalize natural language task or safety requirements and constrain robot behavior through fine-tuning, runtime action filtering, or downstream planning. In these approaches, TL is derived from natural language task or safety instructions and incorporated into the robot's decision-making process. In contrast, we explicitly separate the foundation model's nominal task from an independently imposed operational constraint. In particular, the foundation model continues to pursue its original objective, while our framework enforces the constraint without relying on the model itself to reason about or satisfy it. Moreover, these works consider LTL, whose Boolean semantics do not quantify the degree of satisfaction and cannot directly express explicit time-bounded requirements over real-valued signals. We instead use STL to represent such operational constraints and provide formal satisfaction guarantees during execution.

The work most closely related to our paper is the SafeDec framework \cite{kapoor2025constraineddecodingroboticsfoundation}, which enforces a family of STL specifications at inference time without retraining by modifying action logits based on predicted constraint satisfaction. 
However, this method assesses constraint satisfaction using only a one-step state prediction, and is therefore limited to invariance-type constraints (i.e., safety requirements).
Time-windowed or sequential specifications require reasoning over longer horizons, for which one-step state prediction is insufficient to determine STL satisfaction.

Motivated by the gaps in the literature, we propose a specification-aware action distribution shaping framework that enforces \emph{a rich class of STL constraints} during the execution of a pretrained robotics foundation model, without retraining or modifying its parameters. Our main contributions are: 

\begin{itemize}

\item We introduce a novel constrained optimization problem, whose goal is to minimally modify a foundation model’s action distribution while ensuring a desired level of STL satisfaction, and we provide its closed-form solution.  

\item We show that the proposed approach guarantees STL satisfaction for deterministic dynamical systems.

\item We validate the proposed approach in simulation across multiple environments and under a broad class of STL constraints to demonstrate reliable STL satisfaction while preserving main-task performance.

\end{itemize}

\section{PRELIMINARIES: Signal Temporal Logic}
We specify the operational constraints of the robot using Signal Temporal Logic (STL).

\begin{definition}[Signal Temporal Logic (STL)]
STL \cite{maler2004monitoring} is a specification language for expressing temporal properties of real-valued signals. Its  syntax is:
\begin{flalign}
\phi &:= \top \mid \mu \mid \neg \phi \mid \phi_1 \andltl \phi_2 \mid \phi_1 \orltl \phi_2 \mid \phi_1 \LTLUNTIL_{[a,b]} \phi_2,
\label{eq:STL_syntax}
\end{flalign}
where $0 \leq a \leq b$ are time bounds; $\top$ denotes \emph{true}; $\notltl$, $\andltl$ and $\orltl$ denote the negation, conjunction and disjunction operators, respectively; and $\LTLUNTIL$ denotes the until temporal operator. The symbol $\phi$ denotes an STL formula, while $\mu$ represents a predicate of the form $g(\mathbf{x}) < d$, where $\mathbf{x}: \mathbb{R}_{\geq 0} \rightarrow \mathbb{R}^n$ is a signal, $g:\mathbb{R}^n \rightarrow \mathbb{R}$ is a real-valued function, and $d \in \mathbb{R}$ is a constant. The eventually and always operators are recovered as the standard derived operators $\LTLEVENTUALLY_{[a,b]} \phi := \top \LTLUNTIL_{[a,b]} \phi$ and $\LTLALWAYS_{[a,b]} \phi := \neg \LTLEVENTUALLY_{[a,b]} \neg \phi$.
\label{def:stl}
\end{definition}

For a signal $\mathbf{x}$, let $x_t$ denote the value of $\mathbf{x}$ at time $t$ and $(\mathbf{x},t)$ denote the part of the signal starting at time $t$. The satisfaction of an STL specification over $(\mathbf{x},t)$ is defined recursively as follows:
{\fontsize{9}{12}\selectfont
\begin{flalign*}
&(\mathbf{x}, t) \models \mu \iff g(x_t) < d, & \\
&(\mathbf{x}, t) \models \neg \mu \iff \neg((\mathbf{x}, t) \models \mu), & \\
&(\mathbf{x}, t) \models \phi_1 \land \phi_2 \iff (\mathbf{x}, t) \models \phi_1 \text{ and } (\mathbf{x}, t) \models \phi_2, & \\
&(\mathbf{x}, t) \models \phi_1 \orltl \phi_2 \iff (\mathbf{x}, t) \models \phi_1 \text{ or } (\mathbf{x}, t) \models \phi_2, & \\
&(\mathbf{x}, t) \models \phi_1 \LTLUNTIL_{[a,b]} \phi_2 \iff \exists\, t' \in [t+a, t+b],\ (\mathbf{x}, t') \models \phi_2 \\
&\hspace{11em} \text{and } \forall\, t'' \in [t, t'],\ (\mathbf{x}, t'') \models \phi_1, & \\
&(\mathbf{x}, t) \models \LTLALWAYS_{[a,b]} \phi \iff \forall t' \in [t+a, t+b],\ (\mathbf{x}, t') \models \phi, & \\
&(\mathbf{x}, t) \models \LTLEVENTUALLY_{[a,b]} \phi \iff \exists t' \in [t+a, t+b],\ (\mathbf{x}, t') \models \phi. &
\end{flalign*}
}
The formula $\phi_1 \LTLUNTIL_{[a,b]} \phi_2$ requires that $\phi_1$ holds continuously until $\phi_2$ is satisfied at some point within $[t+a,t+b]$, $\LTLEVENTUALLY_{[a,b]} \phi$ requires that $\phi$ should be satisfied at least once within the interval $[t+a, t+b]$, whereas $\LTLALWAYS_{[a,b]}\phi$ enforces that $\phi$ is satisfied throughout the entire interval. The horizon of an STL formula $\phi$, denoted by $hrz(\phi)$, corresponds to the maximum number of future time steps needed to evaluate whether the specification is satisfied \cite{10.1007/978-3-319-11164-3_19}. 

STL is endowed with a metric called the \emph{robustness degree}, which quantifies how strongly a signal $\mathbf{x}$ satisfies a given specification. The robustness degree is calculated recursively as follows:
{\fontsize{9}{12}\selectfont
\begin{flalign*}
&\rho(\mathbf{x}, \mu, t) = d - g(x_t), & \\
&\rho(\mathbf{x},\notltl \phi, t) = -\rho(\mathbf{x}, \phi, t) , & \\
&\rho(\mathbf{x}, \phi_1 \andltl \phi_2, t) = \min(\rho(\mathbf{x}, \phi_1, t), \rho(\mathbf{x},\phi_2, t)), & \\
&\rho(\mathbf{x}, \phi_1 \orltl \phi_2, t) = \max(\rho(\mathbf{x}, \phi_1, t), \rho(\mathbf{x},\phi_2, t)), & \\
&\rho(\mathbf{x}, \phi_1 \LTLUNTIL_{[a,b]} \phi_2, t) = \sup_{t' \in [t+a,t+b]} \Big(\min\big(\rho(\mathbf{x},\phi_2,t'), & \\
&\hspace{9em} \inf_{t'' \in [t,t']} \rho(\mathbf{x},\phi_1,t'')\big)\Big), & \\
&\rho(\mathbf{x}, \LTLALWAYS_{[a,b]}\phi, t) = \inf_{t' \in [t+a, t+b]} \rho(\mathbf{x}, \phi, t'), & \\
&\rho(\mathbf{x}, \LTLEVENTUALLY_{[a,b]}\phi, t) = \sup_{t' \in [t+a, t+b]} \rho(\mathbf{x}, \phi, t'). &
\end{flalign*}
}
For notational simplicity, we will use $\rho(\mathbf{x}, \phi)$ to denote $\rho(\mathbf{x}, \phi, 0)$ throughout the remainder of the paper.

\section{PROBLEM STATEMENT}
\label{sec:Problem Statement}
\noindent \emph{Robot:} We consider a robot modeled as a discrete-time dynamical system 
\vspace{-2mm}
\begin{equation}
 x_{t+1} = f(x_t, a_t), \nonumber   
\end{equation}
where $x_t \in \mathcal{X} \subseteq \mathbb{R}^n$ is the state at time $t$, $a_t \in \mathcal{A}$ is a discrete action selected from a finite action set $\mathcal{A}$ at time $t$, and $f: \mathcal{X} \times \mathcal{A} \rightarrow \mathcal{X}$ is a known dynamics model. Given an initial state $x_0$ and action sequence $\mathbf{a}_{0:T-1}$, the dynamics model produces the corresponding state trajectory $\mathbf{x}_{0:T}$. 
Moreover, the robot is equipped with onboard sensors that provide observations of the environment. At physical state $x_t$, the robot receives a sensory measurement \[y_t = h(x_t) \in \mathcal{Y},\] where $h: \mathcal{X} \rightarrow \mathcal{Y}$ denotes the observation model mapping physical states to sensor outputs. The observation space $\mathcal{Y}$ is a measurable space representing all admissible sensor measurements, whose dimensionality and structure depend on the sensing modality (e.g., RGB images, LiDAR scans). In general, $\mathcal{Y}$ is high-dimensional and provides partial information about the environment and the robot’s configuration.

\vspace{0.2em}
\noindent \emph{Environment:} We assume that our robot operates in a static, complex indoor environment similar to those used in training and evaluation of robotics foundation models such as SPOC \cite{SPOC}, PoliFormer \cite{Poliformer} and FLaRe \cite{FlaRe}. These environments consist of multiple rooms, furniture, and manipulable objects representative of realistic domestic settings. We assume that a coarse structural layout of the space, including the floor plan and large static obstacles (e.g., walls and fixed furniture), is known, as robots can obtain such maps using standard SLAM-based techniques. However, the precise locations of smaller objects and task-relevant items are unknown. Hence, we characterize these environments as \emph{structurally known but semantically complex and unknown}.

\vspace{0.2em}
\noindent \emph{Mission:} We consider a robot that is tasked with a mission specified in natural language within a semantically complex and a priori unknown environment. Such missions are inherently unstructured and demand reasoning beyond pre-defined tasks and environment models. Robotics foundation models are particularly well-suited for such settings because they can leverage broad contextual knowledge and semantic understanding during decision-making. 

In this work, we consider robotics foundation models that output a probability distribution over a finite discrete action set at each decision step (e.g., \cite{SPOC, FlaRe, Poliformer, pmlr-v270-kim25c}). Let $\mathcal{I}$ denote the natural language instruction given to the foundation model (FM). These models are autoregressive, so we denote by $\xi_t$ the aggregated input history available to the model at time $t$, which may include the sensory measurements $\mathbf{y}_{0:t}$, the previously executed actions $\mathbf{a}_{0:t-1}$, and an encoding of the instruction $\mathcal{I}$. At each time step, the FM produces a discrete probability distribution as
\begin{equation}
\pi_{\mathrm{FM}}(\cdot \mid \xi_t) \in \Delta(\mathcal{A}),
\end{equation}
where $\mathcal{A}$ is the finite action set and 
$\Delta(\mathcal{A}) := \{ p : \mathcal{A} \to [0,1] \mid \sum_{a \in \mathcal{A}} p(a) = 1 \}$ denotes the probability simplex over $\mathcal{A}$. During deployment, an action is sampled from $\pi_{\mathrm{FM}}(\cdot\mid \xi_t)$ and executed on the robot.

\color{black}
\vspace{0.2em}
\noindent \emph{Constraint:} In addition to the natural-language mission given to the FM, we impose explicit operational constraints that must be satisfied during execution. For example, a robot instructed to “find the TV remote” may also need to reach a charging station within a specified time due to battery limitations. This requirement can be expressed as an STL specification such as “eventually reach a charger within 60 time steps.” We therefore encode operational constraints as an STL formula $\phi$.

Since we consider a rich class of STL specifications that may include time-windowed reachability and other horizon-dependent constraints, satisfaction cannot be determined from the next state alone. In contrast to the constrained decoding framework in \cite{kapoor2025constraineddecodingroboticsfoundation}, which focuses on invariance-type specifications whose violation can be detected from single-step transitions, our setting requires reasoning over future state evolution. To evaluate the future impact of a candidate action, we introduce a secondary policy
\begin{equation}
\label{eq:pi_phi}
    \pi_{\phi} : \bar{\mathcal{X}} \times \{0, \dots, T-1\} \rightarrow \mathcal{A},
\end{equation}
that is designed to purely satisfy $\phi$ within its time horizon $T = hrz(\phi)$. In general, $\pi_{\phi}$ cannot always be defined over the physical state $x_t \in \mathcal{X}$ and the raw time index alone, since satisfying $\phi$ may require additional information, such as a history of past states or which sub-formulas of $\phi$ have already been satisfied (e.g., \cite{aksaray2016q, pmlr-v120-venkataraman20a, meng2025tgpotemporalgroundedpolicy}). We therefore let $\bar{x}_t \in \bar{\mathcal{X}}$ denote an augmented state that combines the physical state $x_t$ with such task-progress information, and define $\pi_{\phi}$ over $\bar{\mathcal{X}}$. 
Although the environment is semantically complex and initially unknown, we assume that the structural map (e.g., floor plan and large obstacles) and the regions referenced by $\phi$ are known. This is sufficient to synthesize $\pi_{\phi}$ and evaluate candidate actions by forward propagation.

We now formally define a specification evaluation function that captures the impact of a candidate action on the satisfaction of $\phi$ over the full horizon.

\begin{definition}[Specification Evaluation Function]
Let $\phi$ be an STL specification with horizon $T$, and let $\pi_{\phi}$ denote a policy designed to satisfy $\phi$ over this horizon. For each time step $t \in \{0,\dots,T-1\}$ and action $a_t \in \mathcal{A}$, we define the specification evaluation function
\vspace{-2mm}
\begin{equation}
    J_{\phi}(\mathbf{x}_{0:t}, a_t, \pi_{\phi}) \in \mathbb{R}
\end{equation}
where $\mathbf{x}_{0:t} = (x_0,\dots,x_t)$ denotes the executed state trajectory up to time $t$. This function quantifies the extent to which executing $a_t$ at time $t$ and subsequently following $\pi_{\phi}$ leads to the satisfaction of $\phi$ over the remaining horizon.
\label{def:sef}
\end{definition}

Using $J_{\phi}$, we enforce STL compliance by constraining the modified action distribution $\pi(\cdot \mid \xi_t)$ to place sufficient mass on actions that admit a ``good" continuation under $\pi_{\phi}$. In particular, for a threshold $\delta \in \mathbb{R}$, we impose
\[\mathbb{E}_{a_t \sim \pi(\cdot \mid \xi_t)}
\big[
J_{\phi}\!\big(\mathbf{x}_{0:t},a_t,\pi_{\phi}\big)
\big] \geq \delta,\]

\color{black}

Finally, our objective is to minimally modify the foundation model’s action distribution while enforcing the STL constraint. We quantify the deviation from the FM's distribution using the Kullback–Leibler (KL) divergence $D_{KL}(\pi_1 \,\|\, \pi_2 )$, which is an information-theoretic measure of discrepancy between probability distributions $\pi_1$ and $\pi_2$. This leads to the following constrained optimization problem.

\begin{problem}
\label{problem_formulation}
Given a natural language instruction $\mathcal{I}$, an STL specification $\phi$ with horizon $T$, and a policy
$\pi_{\phi}$ designed to satisfy $\phi$, compute a modified policy
$\pi^*(\cdot \mid \xi_t) \in \Delta(\mathcal{A})$ at each time step $t$ as follows:
\begin{equation}
\pi^*(\cdot \mid \xi_t)
=
\arg\min_{\pi \in \Delta(\mathcal{A})}
D_{\mathrm{KL}}\!\left(\pi(\cdot \mid \xi_t)\,\|\, \pi_{\mathrm{FM}}(\cdot \mid \xi_t)\right)
\end{equation}
subject to
\begin{equation}
\label{optimization_constraint}
\mathbb{E}_{a_t \sim \pi(\cdot \mid \xi_t)}
\big[
J_{\phi}\!\big(\mathbf{x}_{0:t},a_t,\pi_{\phi}\big)
\big] \geq \delta,
\end{equation}
where $\pi_{\mathrm{FM}}(\cdot \mid \xi_t)$ is the action distribution induced by the foundation model, $J_{\phi}$ is the specification evaluation function, and $\delta$ is a user-specified threshold.
\end{problem}

\section{PROPOSED SOLUTION}
In this section, we first explain how the STL-satisfying policy $\pi_{\phi}$ is constructed and how the specification evaluation function $J_{\phi}$ in Def.~\ref{def:sef} is defined. We then present the solution to Problem~\ref{problem_formulation} and introduce the main algorithm that enforces the STL specification during task execution. An overview of the proposed framework is shown in Fig.~\ref{fig:overview}.

\subsection{STL Policy Synthesis}
\label{sec:RL}
As discussed in Section \ref{sec:Problem Statement}, evaluating whether a candidate action promotes STL compliance cannot, in general, be determined solely from its immediate next-state effect. For time-bounded specifications, such as reachability within a finite horizon, the consequence of an action depends on the future trajectory induced after that action is taken. Therefore, we require a policy $\pi_{\phi}$ that is explicitly synthesized to satisfy the STL specification over the remaining horizon. 

A variety of approaches have been proposed for synthesizing control policies under STL constraints, including optimization-based methods \cite{7039363, 7447084}, Control Barrier Function (CBF) based methods \cite{Lindemann2019, buyukkocak2022control, BUYUKKOCAK2024104681, yuksel2026}, and learning-based methods \cite{aksaray2016q, pmlr-v120-venkataraman20a, DeepRLLagrangian, wang_multi_agent_rl, FunnelRewardRL, meng2025tgpotemporalgroundedpolicy}. Optimization-based methods can provide strong guarantees on the resulting trajectories, but often require solving computationally intensive programs whose complexity scales with the specification horizon. CBF-based approaches are computationally efficient at runtime, but typically rely on control-affine dynamics and require careful manual construction of barrier functions tailored to specific classes of STL specifications. Learning-based methods offer greater flexibility for nonlinear dynamics and complex specifications, but require training prior to deployment and typically offer weaker formal guarantees than optimization-based methods. Our framework is agnostic to which of these techniques is used: any method capable of producing a policy that satisfies $\phi$ over the horizon $T$ can be used to synthesize $\pi_{\phi}$, without any modification to the remainder of the framework.

\subsection{Specification Evaluation via $J_{\phi}$}
Once $\pi_{\phi}$ is obtained, it serves as a predictive guide: given a candidate action at time $t$, we evaluate the specification over the trajectory induced by executing that action and subsequently following $\pi_{\phi}$. Intuitively, $\pi_{\phi}$ represents an STL-aware recovery strategy. It allows us to penalize actions that would lead the robot into states from which the specification cannot be satisfied, or from which the expected satisfaction level would drop below the desired threshold, even under an optimal STL-satisfying controller. For example, in a time-bounded reachability constraint such as “eventually reach a charger within 60 time steps,” actions that move the robot away from the charger late in the horizon should be penalized, since STL satisfaction may no longer be possible within the remaining time. This allows the FM’s action distribution to be shaped according to whether each candidate action preserves satisfaction of $\phi$ over the remaining horizon.

While $J_{\phi}$ can represent different notions of specification satisfaction (e.g., robustness-based scores), in this paper, we use a binary feasibility measure obtained from the sign of the STL robustness:
\begin{equation}
    J_{\phi}(\mathbf{x}_{0:t}, a_t, \pi_{\phi}) = \mathbf{1}[\rho(\tilde{\mathbf{x}}_{0:T}, \phi) > 0],
    \label{eq:sef}
\end{equation}
where $\mathbf{1}[\cdot]$ denotes the indicator function, which equals $1$ if the condition inside the brackets is satisfied and $0$ otherwise. Here, $\rho(\cdot,\phi)$ is the STL robustness function, and $\tilde{\mathbf{x}}_{0:T}$ is the complete trajectory obtained by concatenating the executed prefix $\mathbf{x}_{0:t}$ with the future trajectory $\hat{\mathbf{x}}_{t+1:T} = (\hat{x}_{t+1},\dots,\hat{x}_T)$ generated by applying action $a_t$ at time $t$ and subsequently following the policy $\pi_{\phi}$ under the known dynamics $f$. Accordingly, $J_{\phi}=1$ if the candidate action admits a satisfying continuation, and $J_{\phi}=0$ otherwise.

\color{black}

\begin{figure}[t!]
\centering
\vspace{5pt}
\includegraphics[width=\linewidth]{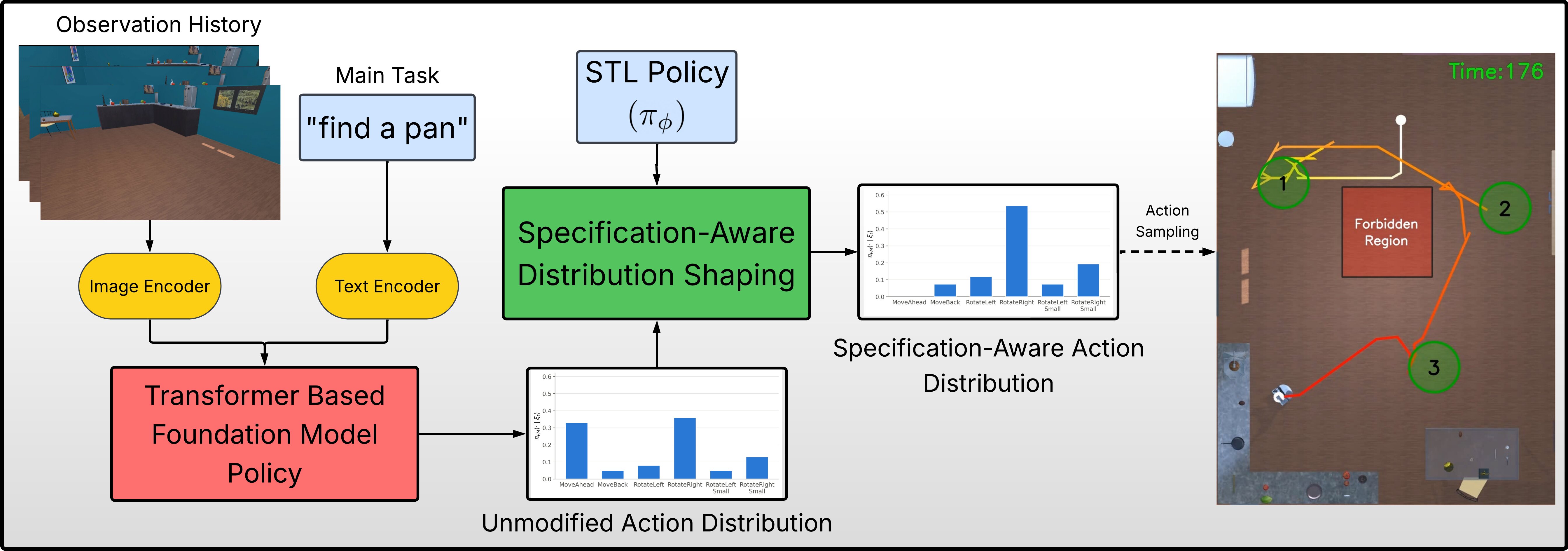}
\caption{Overview of the proposed framework. The foundation model outputs an action distribution from its observation history and the natural language instruction, which the proposed method minimally modifies so that every sampled action preserves satisfaction of $\phi$.}
\label{fig:overview}
\end{figure}

\subsection{Action Distribution Optimization Under STL Constraints}
Our objective is to solve Problem \ref{problem_formulation} at each decision step, which produces a modified action distribution that satisfies the prescribed STL constraint while remaining as close as possible to the original FM's distribution. Minimizing the deviation from the original distribution preserves the causal structure learned during pretraining and avoids a mismatch between the model’s internal state evolution and the executed action sequence, which could otherwise degrade performance. This formulation can be interpreted as projecting the FM’s action distribution onto the set of specification-compliant distributions. A similar idea is used in SafeDec \cite{kapoor2025constraineddecodingroboticsfoundation}, where actions that violate a class of STL specifications are suppressed at the logit level. However, SafeDec evaluates each candidate action using only a one-step state prediction along with the executed trajectory until that point. This is sufficient for invariance-type constraints, whose violation can be detected locally, but not for time-bounded reachability or sequential specifications, whose satisfaction depends on the future trajectory over multiple steps. In contrast, our approach evaluates each candidate action over the remaining specification horizon.

After computing $\pi_{\phi}$ according to the methods mentioned in Sec.~\ref{sec:RL}, considering $J_{\phi}$ as in \eqref{eq:sef}, and setting $\delta = 1$, Problem 1 becomes:
\begin{equation}
\label{eq:objective}
\pi^*(\cdot \mid \xi_t)
=
\arg\min_{\pi \in \Delta(\mathcal{A})}
D_{\mathrm{KL}}\!\left(\pi(\cdot \mid \xi_t)\,\|\, \pi_{\mathrm{FM}}(\cdot \mid \xi_t)\right)
\end{equation}
subject to
\begin{equation}
\label{eq:constraint}
\mathbb{E}_{a_t \sim \pi(\cdot \mid \xi_t)}
\big[
\mathbf{1}[\rho(\tilde{\mathbf{x}}_{0:T}, \phi) > 0]
\big] = 1,
\end{equation}

Here, the Kullback–Leibler divergence is
\begin{equation}
D_{\mathrm{KL}}\!\left(\pi(\cdot \mid \xi_t)\,\|\,\pi_{\mathrm{FM}}(\cdot \mid \xi_t)\right)
=
\sum_{a_i \in \mathcal{A}}
\pi(a_i \mid \xi_t)
\log
\frac{\pi(a_i \mid \xi_t)}{\pi_{\mathrm{FM}}(a_i \mid \xi_t)}.
\end{equation}

Note that Problem~\ref{problem_formulation} is a convex optimization problem, since the KL divergence is convex in $\pi$ and the expectation constraint in \eqref{eq:constraint} is linear.
For notational convenience, define
\begin{equation}
J_{i,t} := J_{\phi}(\mathbf{x}_{0:t}, a_i, \pi_{\phi}) \in \{0,1\}.
\end{equation}
Forming the Lagrangian,
\begin{equation}
\begin{aligned}
\mathcal{L}(\pi,\lambda,\nu)
&=
\sum_{a_i \in \mathcal{A}}
\pi(a_i \mid \xi_t)
\log
\frac{\pi(a_i \mid \xi_t)}{\pi_{\mathrm{FM}}(a_i \mid \xi_t)} \\
&\quad
+ \lambda
\Big(
\sum_{a_i \in \mathcal{A}}
\pi(a_i \mid \xi_t) J_{i,t}
- 1
\Big)
+ \nu
\Big(
\sum_{a_i \in \mathcal{A}}
\pi(a_i \mid \xi_t)
- 1
\Big),
\end{aligned}
\end{equation}
where $\lambda \in \mathbb{R}$ is the multiplier associated with the STL feasibility constraint and $\nu \in \mathbb{R}$ enforces the normalization constraint $\sum_{a_i \in \mathcal{A}} \pi(a_i \mid \xi_t) = 1$, which arises from the fact that $\pi(\cdot \mid \xi_t)$ lies in the probability simplex over $\mathcal{A}$. 

Taking the partial derivative of the Lagrangian $\mathcal{L}(\pi,\lambda,\nu)$ with respect to $\pi(a_i \mid \xi_t)$ and setting it equal to zero yields
\begin{equation}
   \log
\frac{\pi(a_i \mid \xi_t)}{\pi_{\mathrm{FM}}(a_i \mid \xi_t)}
+ 1
+ \lambda J_{i,t}
+ \nu
=
0. 
\end{equation}

Solving for $\pi(a_i \mid \xi_t)$ gives the exponential-form solution
\begin{equation}
  \pi(a_i \mid \xi_t)
=
\pi_{\mathrm{FM}}(a_i \mid \xi_t)
\exp(-\lambda J_{i,t} - 1 - \nu).  
\end{equation}

Enforcing normalization leads to
\begin{equation}
\pi^*(a_i \mid \xi_t; \lambda)
=
\frac{
\pi_{\mathrm{FM}}(a_i \mid \xi_t)
\exp(-\lambda J_{i,t})
}{
\sum_{a_j \in \mathcal{A}}
\pi_{\mathrm{FM}}(a_j \mid \xi_t)
\exp(-\lambda J_{j,t})
}.
\label{eq:exp_tilt_hard}
\end{equation}

Since $J_{i,t} \in \{0,1\}$, taking the limit $\lambda \to -\infty$ suppresses all actions with $J_{i,t} = 0$, and the solution reduces to
\begin{equation}
\pi^*(a_i \mid \xi_t)
=
\begin{cases}
\dfrac{\pi_{\mathrm{FM}}(a_i \mid \xi_t)}
{\sum_{a_j \in S_t} \pi_{\mathrm{FM}}(a_j \mid \xi_t)}
& \text{if } a_i \in S_t, \\
0 & \text{otherwise,}
\end{cases}
\end{equation}
where $S_t := \{ a_i \mid J_{i,t} = 1 \}$.

Algorithm~\ref{alg:main} summarizes the procedure. At each step, the FM is queried for $\pi_{\mathrm{FM}}$ (line 3), and the main task terminates when the model samples the \texttt{end} token. If \emph{STLDone} holds, i.e., $\phi$ is already satisfied, the action is sampled from $\pi_{\mathrm{FM}}$ (line 5). Otherwise, while $t < T$, $J_{\phi}$ is evaluated for each candidate action by forward propagation (lines 8-9), and since the dynamics and $\pi_{\phi}$ are deterministic, each candidate induces a unique continuation. If Problem~\ref{problem_formulation} is feasible, the action is sampled from $\pi^*$ (lines 10-12). Problem~\ref{problem_formulation} becomes infeasible when $S_t = \emptyset$, i.e., $J_{i,t} = 0$ for every $a_i \in \mathcal{A}$. When the forward-propagation model matches the true dynamics, this cannot occur since the action prescribed by $\pi_{\phi}$ always admits a satisfying continuation, but under model mismatch the true state may deviate from the prediction so that no such continuation may remain. In that case, the algorithm falls back to $\pi_{\phi}$ (line 14). After the horizon expires, actions are sampled from $\pi_{\mathrm{FM}}$ (line 16). If the main task ends before $\phi$ is satisfied, the algorithm follows $\pi_{\phi}$ until it is satisfied (lines 21-26), and  $T_{\max} > T$ guarantees termination.

\begin{proposition}
\label{prop:satisfaction_guarantee}
Assume that (i) the true robot dynamics satisfy $x_{t+1}=f(x_t,a_t)$ (i.e., no model mismatch between the forward-propagation model and the physical system), and (ii) the STL specification $\phi$ is satisfiable from the initial state $x_0$ under the policy $\pi_{\phi}$. If actions are sampled according to Alg.~\ref{alg:main}, the resulting closed-loop trajectory $\mathbf{x}_{0:T}$ is guaranteed to satisfy $\phi$. 
\end{proposition}

\vspace{3pt}
\begin{algorithm2e}[t]
\small{
\caption{Specification-Aware Action Distribution Optimization}
\label{alg:main}
\SetKwInOut{Input}{Input}

\Input{Foundation model $\mathcal{F}$ queried with natural language instruction $\mathcal{I}$, STL specification $\phi$ with horizon $T$, STL-satisfying policy $\pi_{\phi}$, STL evaluation function $J_{\phi}$, discrete action set $\mathcal{A}$, global time limit $T_{\max}$, initial state $x_0$}

Initialize $t \leftarrow 0$, initial augmented state $\bar{x}_0$, set $\mathbf{x}_{0:0} \leftarrow [x_0]$

\While{\textbf{not} \emph{MainTaskDone}$(x_t,\mathcal{I})$ \textbf{and} $t < T_{\max}$}{
    Query $\mathcal{F}$ to obtain $\pi_{\mathrm{FM}}(\cdot\mid \xi_t)$
    
    \eIf{\emph{STLDone}$(\mathbf{x}_{0:t},\phi)$}{
        Sample $a_t \sim \pi_{\mathrm{FM}}(\cdot\mid \xi_t)$
    }{
    \eIf{$t < T$}{
        \ForEach{$a_i \in \mathcal{A}$}{
        Compute $J_{\phi}(\mathbf{x}_{0:t}, a_i, \pi_{\phi})$
        }
        
        \eIf{\emph{Problem}~\ref{problem_formulation} is feasible}{
            Solve \emph{Problem}~\ref{problem_formulation} to obtain $\pi^*(\cdot\mid \xi_t)$
            
            Sample $a_t \sim \pi^*(\cdot\mid \xi_t)$
        }{
            Sample $a_t \leftarrow \pi_{\phi}(\bar{x}_t, t)$
        }
    }{Sample $a_t \sim \pi_{\mathrm{FM}}(\cdot\mid \xi_t)$}

    }
    
    Execute $a_t$, observe $x_{t+1}$, update $\bar{x}_{t+1}$
    
    Append $x_{t+1}$ to $\mathbf{x}_{0:t}$ to obtain $\mathbf{x}_{0:t+1}$

    $x_t \leftarrow x_{t+1}$, $\bar{x}_t \leftarrow \bar{x}_{t+1}$
    
    $t \leftarrow t+1$
}

\While{\textbf{not} \emph{STLDone}$(\mathbf{x}_{0:t},\phi)$ \textbf{and} $t < T$}{
    Sample $a_t \leftarrow \pi_{\phi}(\bar{x}_t, t)$
    
    Execute $a_t$, observe $x_{t+1}$, update $\bar{x}_{t+1}$
    
    Append $x_{t+1}$ to $\mathbf{x}_{0:t}$ to obtain $\mathbf{x}_{0:t+1}$

    $x_t \leftarrow x_{t+1}$, $\bar{x}_t \leftarrow \bar{x}_{t+1}$
    
    $t \leftarrow t+1$
}
}
\end{algorithm2e}

\begin{proof}
This can be shown by examining two cases:

\noindent\emph{Case 1 (STL completion occurs during the main-task loop):}
For any time $t < T$ when  $\phi$ remains unsatisfied, Alg.~\ref{alg:main}
samples $a_t \sim \pi^*(\cdot \mid \xi_t)$. Under the hard feasibility constraint and the fact that $J_{\phi}(\mathbf{x}_{0:t},a,\pi_{\phi})\in\{0,1\}$, the optimizer constructs $\pi^*(\cdot\mid \xi_t)$ so that $\pi^*(a\mid \xi_t)>0$ only if $J_{\phi}(\mathbf{x}_{0:t},a,\pi_{\phi})=1$. Hence, the sampled action $a_t$ admits a satisfying continuation under $\pi_{\phi}$ from the current history. Under the assumption of no model mismatch, the continuation assessed by $J_{\phi}$ is consistent with the true state evolution. Repeating this reasoning at each $t<T$ preserves feasibility step by step until $\phi$ becomes satisfied within the horizon. 

\vspace{0.5em}
\noindent\emph{Case 2 (STL completion occurs after the main-task loop):}
Suppose that the main task terminates at some $t'<T$ before $\phi$ is satisfied. By construction, the executed prefix $\mathbf{x}_{0:t'}$ is generated by Alg.~\ref{alg:main} and actions up to time $t'-1$ are sampled from $\pi^*(\cdot\mid \xi_t)$ whenever STL enforcement is active. 
Thus, the same feasibility-preservation argument as in Case~1 ensures that following $\pi_{\phi}$ at time $t'$ can still satisfy the specification within the remaining horizon.
Algorithm~\ref{alg:main} then switches to sampling from $\pi_{\phi}$ for the remaining steps, and since $\pi_{\phi}$ is designed to satisfy $\phi$ over the horizon and a valid satisfying continuation exists from $\mathbf{x}_{0:t'}$, following $\pi_{\phi}$ results in the satisfaction of $\phi$.
\end{proof}

\section{SIMULATION RESULTS}
\subsection{Implementation Details}
\label{sec:implementation}
\noindent\textbf{Foundation Models: }We evaluate the proposed framework using three robotics foundation models, SPOC~\cite{SPOC}, PoliFormer~\cite{Poliformer}, and FLaRe~\cite{FlaRe}, deployed on a Stretch RE-1 robot in AI2-THOR~\cite{kolve2022ai2thorinteractive3denvironment}, a high-fidelity simulator providing photorealistic household environments with diverse layouts and Objaverse assets~\cite{objaverse}. All three are transformer-based policies that output a categorical distribution over a discrete action set conditioned on RGB observations and natural language instructions. SPOC is trained via imitation learning on shortest-path expert trajectories, PoliFormer via on-policy reinforcement learning from scratch, and FLaRe by fine-tuning SPOC with reinforcement learning, so the three span distinct training paradigms. These models support tasks such as object-goal navigation (\emph{``find an alarm clock''}), room traversal (\emph{``visit every room in this house''}), and attribute-conditioned navigation (\emph{``find the largest apple''}).

\noindent\textbf{Robot Model and Environment:} 
The robot is modeled as a point mass following discrete unicycle dynamics. The navigation actions shared by all three FMs move the robot $0.2$\,m forward or backward (\texttt{MoveAhead}, \texttt{MoveBack}) or rotate it in place by $30^\circ$ (\texttt{RotateLeft}, \texttt{RotateRight}) or $6^\circ$ (\texttt{RotateLeftSmall}, \texttt{RotateRightSmall}). We consider a 2D environment with a known coarse structural layout, which is represented by a 2D occupancy map shown in Fig.~\ref{fig:simulation_envs}. This representation is used both to synthesize the STL policy and to perform forward propagation at deployment.

\begin{figure}[b!]
    \centering
    \begin{subfigure}[b]{0.5\linewidth}
        \centering
        \includegraphics[height=3.4cm]{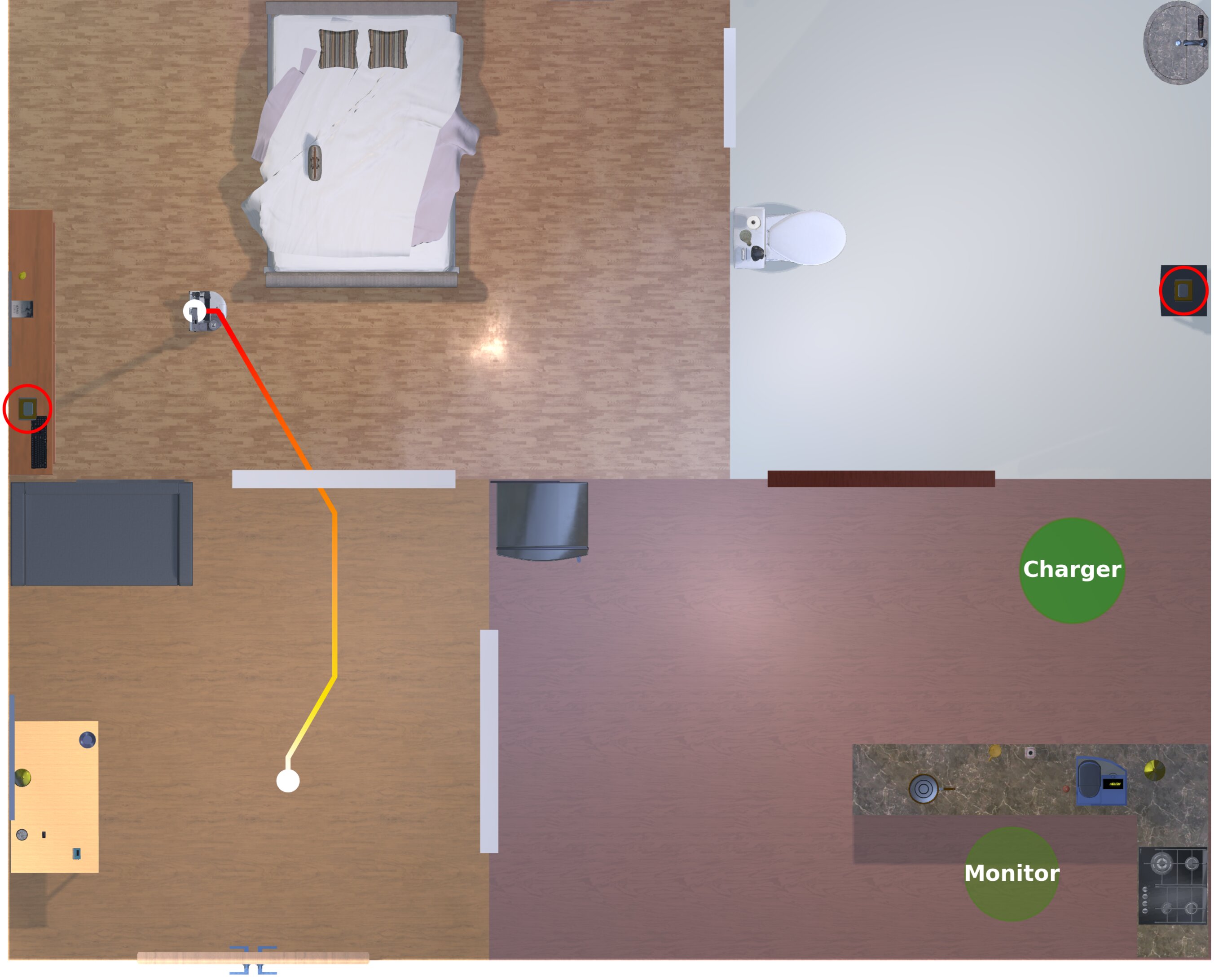}
    \end{subfigure}
    \hspace{-0.2cm}
    \begin{subfigure}[b]{0.49\linewidth}
        \centering
        \includegraphics[height=3.4cm]{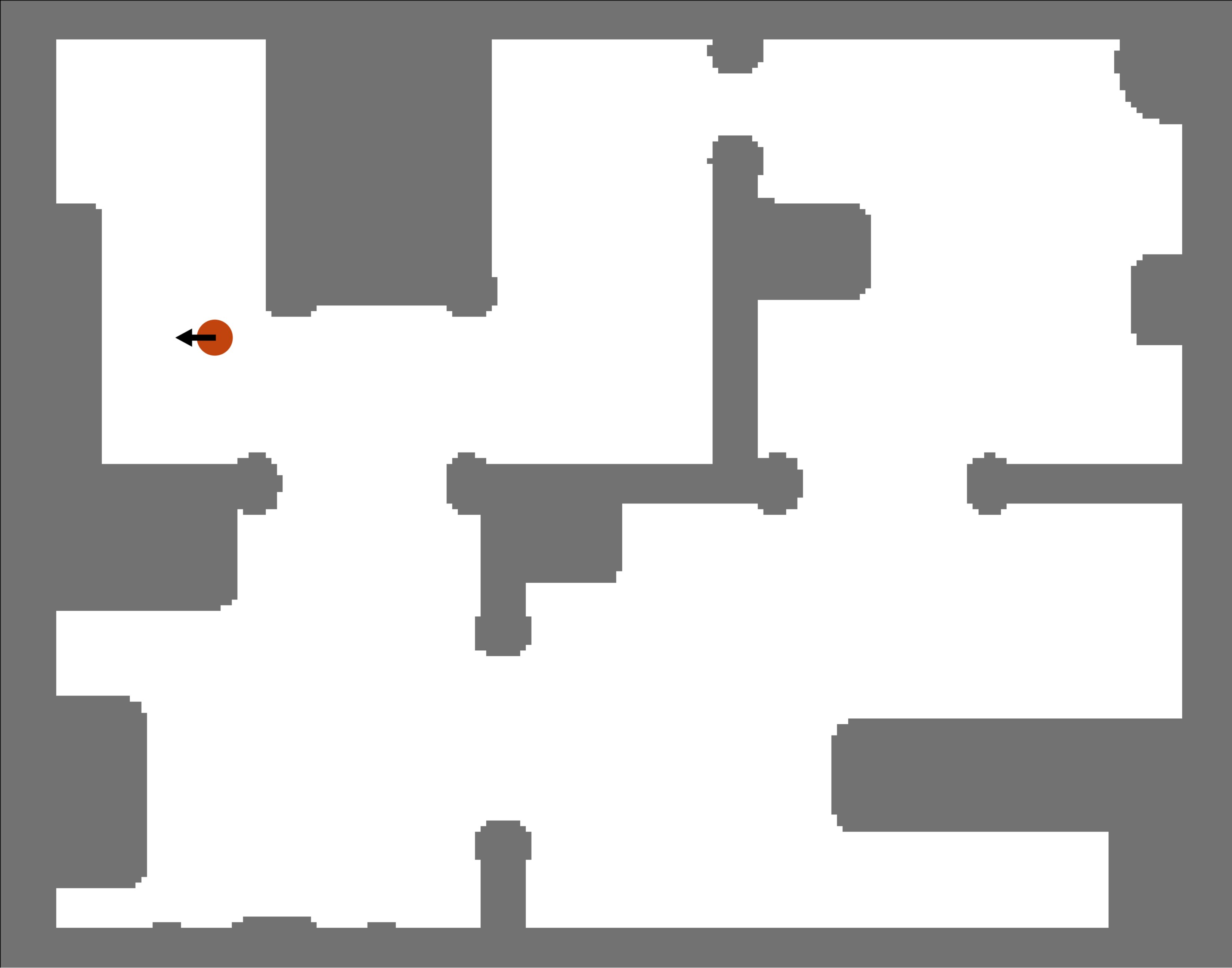}
    \end{subfigure}
    \caption{Top-down view of an AI2-THOR house environment (left), showing the unconstrained FLaRe trajectory for Case 1, and the corresponding 2D occupancy map abstraction - gray: walls and static obstacles, white: free space, red: robot - used for training and forward propagation (right).} 
    \label{fig:simulation_envs}
\end{figure}

\noindent\textbf{STL Policy Synthesis:} As discussed in Sec.~\ref{sec:RL}, the proposed method is agnostic to STL policy synthesis methods. In our case studies, we synthesize $\pi_{\phi}$ with two reinforcement learning methods: Temporal Grounded Policy Optimization (TGPO)~\cite{meng2025tgpotemporalgroundedpolicy} trained with Proximal Policy Optimization (PPO)~\cite{schulman2017proximalpolicyoptimizationalgorithms}, and the funnel-based reward shaping \cite{FunnelRewardRL} trained with a Double Deep Q-Network (DDQN)~\cite{vanhasselt2015deepreinforcementlearningdouble}. In both cases, actions are selected greedily at deployment, so that $\pi_{\phi}$ is deterministic. 

Finally, we set the termination bound of Alg.~\ref{alg:main} to $T_{\max} = 500$ steps.

\subsection{Case 1: Behavior under varying STL constraints}
We demonstrate how the proposed framework alters the behavior of an FM under different STL specifications without retraining it. We use FLaRe~\cite{FlaRe} as the FM and synthesize $\pi_{\phi}$ using TGPO~\cite{meng2025tgpotemporalgroundedpolicy}. We consider two STL constraints:
\begin{align*}
    \phi_1 = {}& \LTLEVENTUALLY_{[0,150]} \LTLALWAYS_{[0,10]}(Charger) \andltl \LTLALWAYS_{[0,160]} \notltl (Bedroom),\\
    \phi_2 = {}& \left(\notltl Charger\right) \LTLUNTIL_{[0,150]} (Stove) \\ & \andltl \LTLEVENTUALLY_{[0,150]} \LTLALWAYS_{[0,10]}(Charger)
            \andltl \LTLALWAYS_{[0,160]} \notltl (Bathroom).
\end{align*}
$\phi_1$ requires the robot to reach the charging station within $150$ time steps and to remain docked for $10$ consecutive time steps, while never entering the bedroom. $\phi_2$ retains the same docking requirement but additionally requires the robot to monitor the stove before docking, and forbids entering the bathroom instead of the bedroom. The monitoring predicate is satisfied when the robot is positioned at the kitchen counter and oriented toward the stove.

In both cases, the FLaRe model is queried with the same object-finding instruction \emph{``find an alarm clock''}. To isolate the effect of the STL specification, two alarm clocks are placed in the environment, one in the bedroom and one in the bathroom, as shown in Fig.~\ref{fig:Case-1}. Since each STL restricts exactly one of these two rooms, the main task remains achievable under both constraints. 

The trajectory under the unmodified policy is shown in Fig.~\ref{fig:simulation_envs} (left), where no constraint is imposed on the sampled actions and the robot proceeds directly to the bedroom and finds the alarm clock there. Figure~\ref{fig:Case-1} compares this behavior with the trajectories obtained under $\phi_1$ and $\phi_2$. Under $\phi_1$, the proposed framework prevents the robot from entering the bedroom, so it instead reaches the charging station, remains docked for the required interval, and then searches the bathroom, where it locates the alarm clock. Under $\phi_2$, the until operator blocks the docking action until the stove has been monitored, so the robot first navigates to the kitchen counter and orients toward the stove before proceeding to the charging station. Since the bathroom is now restricted, it searches the remaining areas of the house and locates the alarm clock in the bedroom. In all three executions, the instruction and the FM are identical, and the divergence in behavior arises solely from the enforced STL constraint.

\begin{figure}[b!]
    \centering
    \begin{subfigure}[b]{0.499\linewidth}
        \centering
        \includegraphics[width=\linewidth]{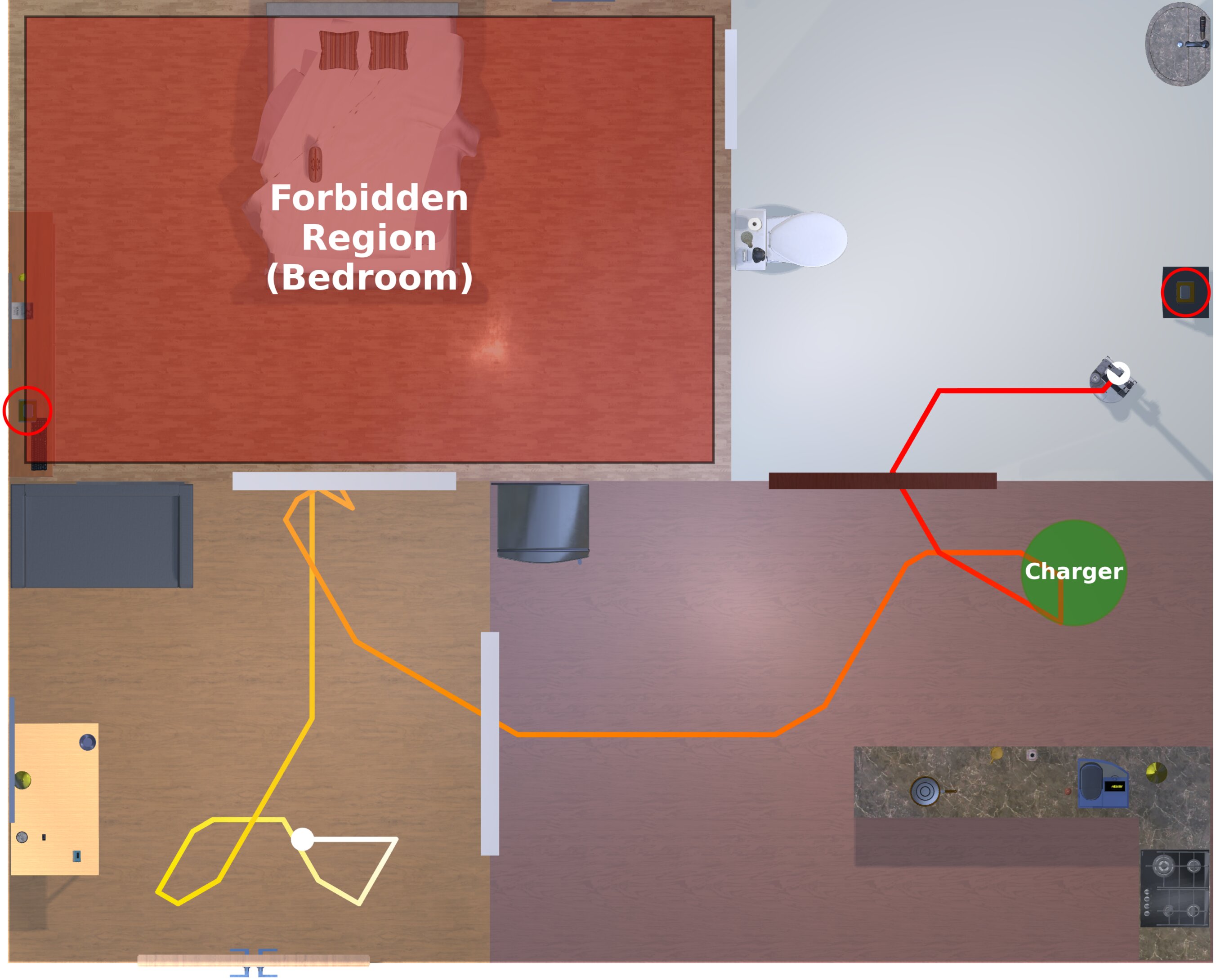}
    \end{subfigure}
    \hspace{-0.2cm}
    \begin{subfigure}[b]{0.499\linewidth}
        \centering
        \includegraphics[width=\linewidth]{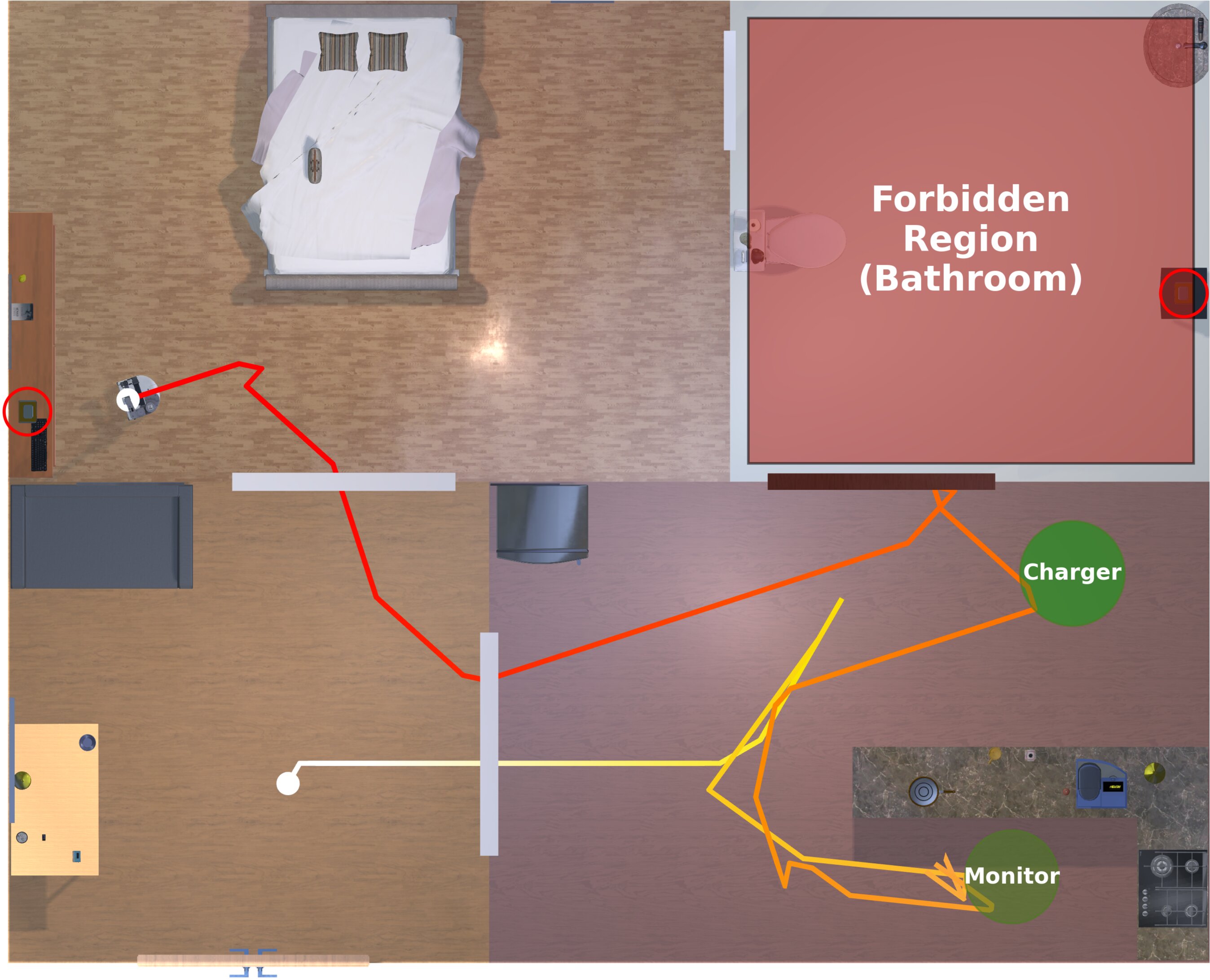}
    \end{subfigure}
    \caption{A robot is instructed to ``find an alarm clock'' under two STL specifications, with the target objects shown as red circles and trajectories colored from white (start) to red (end). [left] Execution under $\phi_1$, which requires docking at the charger (green) for $10$ consecutive steps and forbids entering the bedroom. [right] Execution under $\phi_2$, which also requires monitoring the stove before docking and forbids entering the bathroom.}
    \label{fig:Case-1}
\end{figure}

\subsection{Case 2: Performance across models and STL policies}
In this case study, we demonstrate that the proposed framework is agnostic to both the policy synthesis method and the underlying foundation model. We evaluate both synthesis methods in Sec.~\ref{sec:implementation} with all three foundation models, yielding six configurations under the following specification:
\begin{equation}
\begin{aligned}
    \phi_3 = &\LTLEVENTUALLY_{[0,180]}(Region_1) \andltl \LTLEVENTUALLY_{[0,180]}(Region_2) \andltl \LTLEVENTUALLY_{[0,180]}(Region_3) \notag \\&\andltl \LTLALWAYS_{[0,180]} \notltl (Forbidden\ Regions) \notag
\end{aligned}
\end{equation}

This specification requires the robot to visit all three designated regions within $180$ time steps, in any order, while avoiding both forbidden regions throughout the entire horizon. Since the three reachability requirements share a common time window, the specification does not prescribe a visiting order, and any of the six orderings satisfies $\phi_3$. The FMs are queried with the instruction \emph{``find a laptop''}, where the laptop is placed on the table in the bedroom. Figure~\ref{fig:Case-2} shows the house layout used in this case study, including the three designated regions, the forbidden regions, and the laptop marked by a red circle, together with a representative execution under the unmodified SPOC policy (left) and under the proposed framework with the funnel-based policy (right). Under the unmodified policy, the robot only navigates toward the target object and violates the specification, whereas the proposed framework enforces $\phi_3$ while the model continues to search for the laptop.

\begin{figure}[htb!]
    \centering
    \begin{subfigure}[b]{0.495\linewidth}
        \centering
        \includegraphics[width=\linewidth]{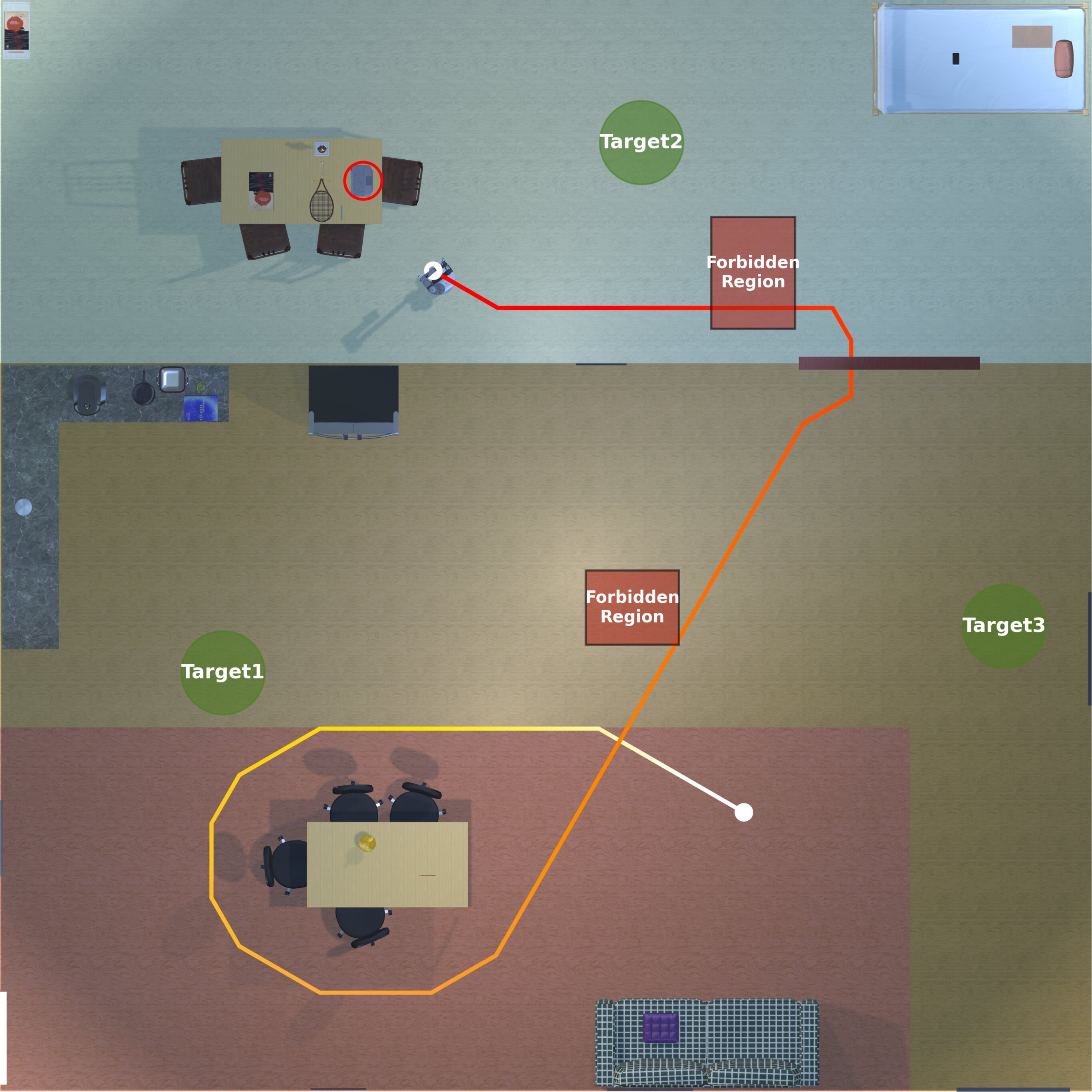}
    \end{subfigure}
    \hspace{-0.2cm}
    \begin{subfigure}[b]{0.495\linewidth}
        \centering
        \includegraphics[width=\linewidth]{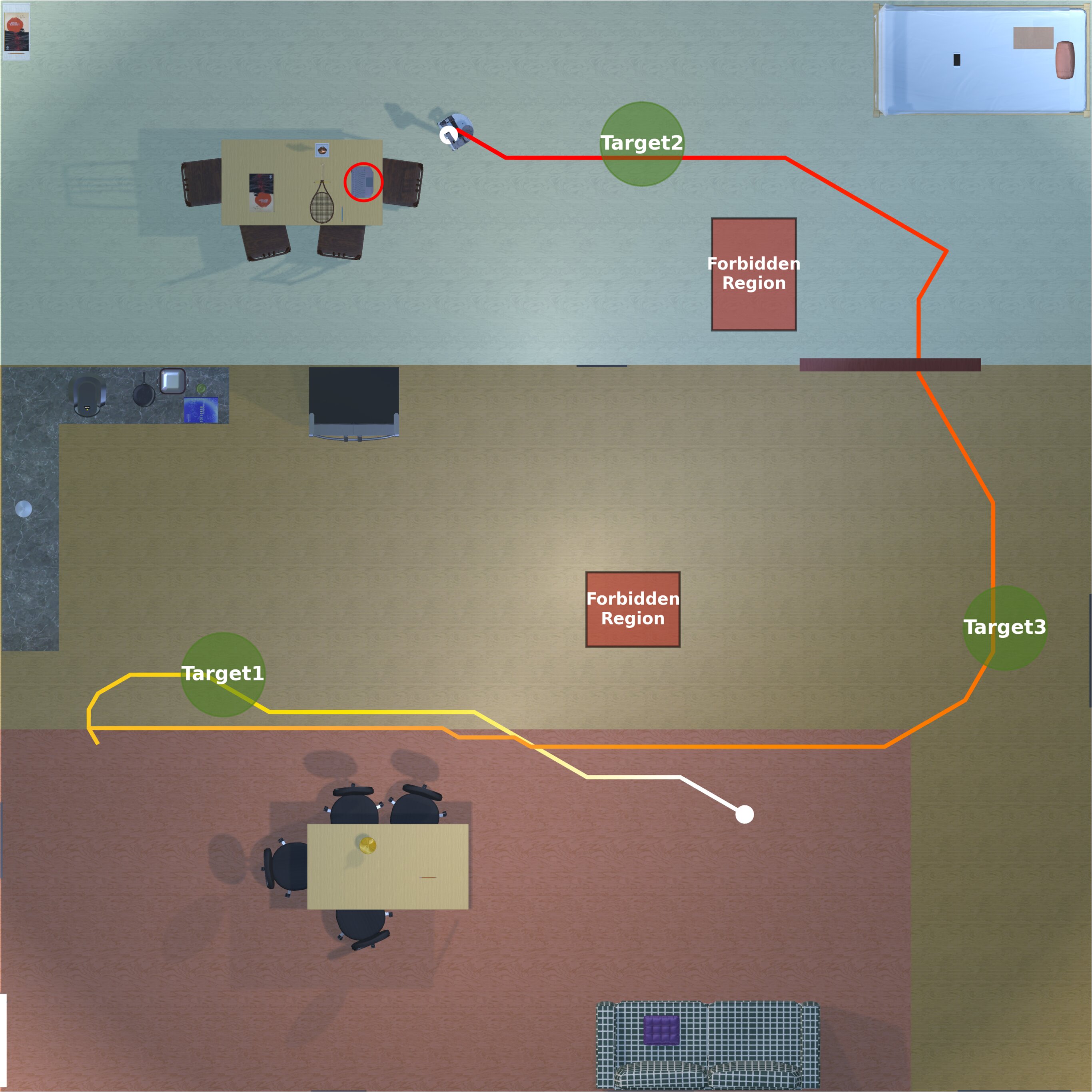}
    \end{subfigure}
    \caption{Trajectory comparison for Case 2 with the target laptop marked by a red circle: execution under (left) the unmodified SPOC policy, (right) the proposed framework with the funnel-based policy.}
    \label{fig:Case-2}
\end{figure}

We evaluate each configuration over $200$ independent simulation runs, and the results are summarized in Table~\ref{tab:results_case2}. The proposed framework achieves a $100\%$ STL satisfaction rate in every configuration, consistent with Prop.~\ref{prop:satisfaction_guarantee}, independently of the foundation model and synthesis method. In contrast, the unmodified policies violate $\phi_3$ in all episodes as expected, since none of the models has any notion of the specification. 

These results reveal two consistent trends. First, for both synthesis methods, the reduction in main task performance is largest for SPOC, smaller for PoliFormer, and smallest for FLaRe. The two foundation models trained with reinforcement learning are therefore considerably more robust to the distribution shift due to enforcing $\phi$ than the imitation-trained SPOC\footnote{Imitation-trained policies are known to struggle in states outside their demonstrations~\cite{FlaRe}.}. Among them, FLaRe, which combines imitation pretraining with reinforcement learning fine-tuning, is the most robust. Second, TGPO-based policies incur a larger reduction than funnel-based policies for all three models. A plausible explanation is that the two synthesis methods induce different degrees of conservatism, so that the set of actions admitting a satisfying continuation is smaller under TGPO, which in turn requires a more substantial modification of the FM's action distribution. Such trade-offs between constraint enforcement and task performance have also been reported in prior constrained decoding frameworks~\cite{kapoor2025constraineddecodingroboticsfoundation, park2025grammaraligneddecoding}. These observations indicate that the choice of $\pi_{\phi}$ directly affects how strongly the model's action distribution is reshaped, even when the STL satisfaction guarantee remains unchanged. A video demonstrating both case studies is available at: \url{https://youtu.be/CFS7-49aVwQ}
\color{black}

\begin{table}[htb!]
\centering
\caption{Empirical success rates over 200 independent simulation runs for Case~2 ($\phi_3$).}
\label{tab:results_case2}
\begin{tabular*}{\columnwidth}{@{\extracolsep{\fill}}lcccccc@{}}
\toprule
& \multicolumn{2}{c}{\textbf{SPOC}} 
& \multicolumn{2}{c}{\textbf{PoliFormer}} 
& \multicolumn{2}{c}{\textbf{FLaRe}} \\
\cmidrule(lr){2-3} \cmidrule(lr){4-5} \cmidrule(lr){6-7}
\textbf{Method} 
& STL & Main 
& STL & Main 
& STL & Main \\
\midrule
Unmodified & 0   & 81.5 & 0   & 100.0 & 0   & 100.0 \\
Funnel     & 100 & 71.5 & 100 & 97.5  & 100 & 99.0  \\
TGPO       & 100 & 59.0 & 100 & 85.0  & 100 & 95.5  \\
\bottomrule
\end{tabular*}
\end{table}

\vspace{-4mm}
\section{CONCLUSION}
We proposed a specification-aware action distribution optimization framework that enforces rich STL constraints during the execution of pretrained robotics foundation models without retraining them. By solving a constrained optimization problem at each decision step, the method modifies the model's action distribution to guarantee satisfaction of time-windowed, nested, and safety-critical specifications while remaining as close as possible to the original policy. Under the assumptions of no model mismatch and that the STL-satisfying policy satisfies the specification from the initial state, we showed that the resulting trajectory is guaranteed to satisfy the constraint. Simulations demonstrated that the same instruction can yield distinct, specification-compliant behaviors under different constraints, and that the framework enforces the specification across different foundation models and policy synthesis methods while largely preserving main-task performance. Future work will focus on improving scalability and relaxing modeling assumptions.

\bibliographystyle{IEEEtran}
\bibliography{references}

@inproceedings{aksaray2016q,
  title={Q-learning for robust satisfaction of signal temporal logic specifications},
  author={Aksaray, Derya and Jones, Austin and Kong, Zhaodan and Schwager, Mac and Belta, Calin},
  booktitle={IEEE 55th Conf. on Decision and Control (CDC)},
  pages={6565--6570},
  year={2016},
  organization={IEEE}
}

@inproceedings{buyukkocak2022control,
  title={Control barrier functions with actuation constraints under signal temporal logic specifications},
  author={Buyukkocak, Ali Tevfik and Aksaray, Derya and Yaz{\i}c{\i}o{\u{g}}lu, Yasin},
  booktitle={European Control Conf. (ECC)},
  pages={162--168},
  year={2022},
  organization={IEEE}
}

@article{BUYUKKOCAK2024104681,
title = {Sequential control barrier functions for mobile robots with dynamic temporal logic specifications},
journal = {Robotics and Autonomous Systems},
volume = {176},
pages = {104681},
year = {2024},
issn = {0921-8890},
doi = {https://doi.org/10.1016/j.robot.2024.104681},
url = {},
author = {Ali Tevfik Buyukkocak and Derya Aksaray and Yasin Yazıcıoğlu}
}

@book{baier2008principles,
  title={Principles of model checking},
  author={Baier, Christel and Katoen, Joost-Pieter},
  year={2008},
  publisher={MIT press}
}

@inproceedings{maler2004monitoring,
  title={Monitoring temporal properties of continuous signals},
  author={Maler, Oded and Nickovic, Dejan},
  booktitle={Int. symposium on formal techniques in real-time and fault-tolerant systems},
  pages={152--166},
  year={2004},
  organization={Springer}
}

@ARTICLE{Lindemann2019,

  author={Lindemann, Lars and Dimarogonas, Dimos V.},

  journal={IEEE Control Systems Letters}, 

  title={Control Barrier Functions for Signal Temporal Logic Tasks}, 

  year={2019},

  volume={3},

  number={1},

  pages={96-101},

  doi={10.1109/LCSYS.2018.2853182}}

@ARTICLE{FunnelRewardRL,
  author={Saxena, Naman and Gorantla, Sandeep and Jagtap, Pushpak},
  journal={IEEE Robotics and Automation Letters}, 
  title={Funnel-Based Reward Shaping for Signal Temporal Logic Tasks in Reinforcement Learning}, 
  year={2024},
  volume={9},
  number={2},
  pages={1373-1379},
  doi={10.1109/LRA.2023.3341775}}

@ARTICLE{DeepRLLagrangian,
  author={Ikemoto, Junya and Ushio, Toshimitsu},
  journal={IEEE Access}, 
  title={Deep Reinforcement Learning Under Signal Temporal Logic Constraints Using Lagrangian Relaxation}, 
  year={2022},
  volume={10},
  number={},
  pages={114814-114828},
  doi={10.1109/ACCESS.2022.3218216}}

@InProceedings{pmlr-v120-venkataraman20a,
  title = 	 {Tractable Reinforcement Learning of Signal Temporal Logic Objectives},
  author =       {Venkataraman, Harish and Aksaray, Derya and Seiler, Peter},
  booktitle = 	 {Proc. of the 2nd Conf. on Learning for Dynamics and Control},
  pages = 	 {308--317},
  year = 	 {2020},
  editor = 	 {},
  volume = 	 {120},
  series = 	 {Proc. of Machine Learning Research},
  month = 	 {10--11 Jun},
  publisher =    {PMLR},
  url = 	 {}
}

@INPROCEEDINGS{7039363,
  author={Raman, Vasumathi and Donzé, Alexandre and Maasoumy, Mehdi and Murray, Richard M. and Sangiovanni-Vincentelli, Alberto and Seshia, Sanjit A.},
  booktitle={53rd IEEE Conf. on Decision and Control}, 
  title={Model predictive control with signal temporal logic specifications}, 
  year={2014},
  volume={},
  number={},
  pages={81-87},
  doi={10.1109/CDC.2014.7039363}}

@INPROCEEDINGS{7447084,
  author={Sadraddini, Sadra and Belta, Calin},
  booktitle={53rd Annual Allerton Conf. on Communication, Control, and Computing (Allerton)}, 
  title={Robust temporal logic model predictive control}, 
  year={2015},
  volume={},
  number={},
  pages={772-779},
  doi={10.1109/ALLERTON.2015.7447084}}

@misc{kapoor2025constraineddecodingroboticsfoundation,
      title={Constrained Decoding for Robotics Foundation Models}, 
      author={Parv Kapoor and Akila Ganlath and Michael Clifford and Changliu Liu and Sebastian Scherer and Eunsuk Kang},
      year={2025},
      eprint={2509.01728},
      archivePrefix={arXiv},
      primaryClass={cs.RO},
      url={}, 
}

@misc{park2025grammaraligneddecoding,
      title={Grammar-Aligned Decoding}, 
      author={Kanghee Park and Jiayu Wang and Taylor Berg-Kirkpatrick and Nadia Polikarpova and Loris D'Antoni},
      year={2025},
      eprint={2405.21047},
      archivePrefix={arXiv},
      primaryClass={cs.AI},
      url={}, 
}

@misc{kolve2022ai2thorinteractive3denvironment,
      title={AI2-THOR: An Interactive 3D Environment for Visual AI}, 
      author={Eric Kolve and Roozbeh Mottaghi and Winson Han and Eli VanderBilt and Luca Weihs and Alvaro Herrasti and Matt Deitke and Kiana Ehsani and Daniel Gordon and Yuke Zhu and Aniruddha Kembhavi and Abhinav Gupta and Ali Farhadi},
      year={2022},
      eprint={1712.05474},
      archivePrefix={arXiv},
      primaryClass={cs.CV},
      url={}, 
}

@misc{hu2024generalpurposerobotsfoundationmodels,
      title={Toward General-Purpose Robots via Foundation Models: A Survey and Meta-Analysis}, 
      author={Yafei Hu and Quanting Xie and Vidhi Jain and Jonathan Francis and Jay Patrikar and Nikhil Keetha and Seungchan Kim and Yaqi Xie and Tianyi Zhang and Hao-Shu Fang and Shibo Zhao and Shayegan Omidshafiei and Dong-Ki Kim and Ali-akbar Agha-mohammadi and Katia Sycara and Matthew Johnson-Roberson and Dhruv Batra and Xiaolong Wang and Sebastian Scherer and Chen Wang and Zsolt Kira and Fei Xia and Yonatan Bisk},
      year={2024},
      eprint={2312.08782},
      archivePrefix={arXiv},
      primaryClass={cs.RO},
      url={}, 
}

@article{doi:10.1177/02783649241281508,
author = {Roya Firoozi and Johnathan Tucker and Stephen Tian and Anirudha Majumdar and Jiankai Sun and Weiyu Liu and Yuke Zhu and Shuran Song and Ashish Kapoor and Karol Hausman and Brian Ichter and Danny Driess and Jiajun Wu and Cewu Lu and Mac Schwager},
title ={Foundation models in robotics: Applications, challenges, and the future},

journal = {The Int. Journal of Robotics Research},
volume = {44},
number = {5},
pages = {701-739},
year = {2025},
doi = {10.1177/02783649241281508},

URL = { 

},

}

@InProceedings{SPOC,
    author    = {Ehsani, Kiana and Gupta, Tanmay and Hendrix, Rose and Salvador, Jordi and Weihs, Luca and Zeng, Kuo-Hao and Singh, Kunal Pratap and Kim, Yejin and Han, Winson and Herrasti, Alvaro and Krishna, Ranjay and Schwenk, Dustin and VanderBilt, Eli and Kembhavi, Aniruddha},
    title     = {SPOC: Imitating Shortest Paths in Simulation Enables Effective Navigation and Manipulation in the Real World},
    booktitle = {Proc. of the IEEE/CVF Conf. on Computer Vision and Pattern Recognition (CVPR)},
    month     = {June},
    year      = {2024},
    pages     = {16238-16250}
}

@INPROCEEDINGS{FlaRe,
  author={Hu, Jiaheng and Hendrix, Rose and Farhadi, Ali and Kembhavi, Aniruddha and Martín-Martín, Roberto and Stone, Peter and Zeng, Kuo-Hao and Ehsani, Kiana},
  booktitle={IEEE Int. Conf. on Robotics and Automation (ICRA)}, 
  title={FLaRe: Achieving Masterful and Adaptive Robot Policies with Large-Scale Reinforcement Learning Fine-Tuning}, 
  year={2025},
  volume={},
  number={},
  pages={3617-3624},
  doi={10.1109/ICRA55743.2025.11127934}}

@InProceedings{Poliformer,
  title = 	 {PoliFormer: Scaling On-Policy RL with Transformers Results in Masterful Navigators},
  author =       {Zeng, Kuo-Hao and Zhang, Zichen and Ehsani, Kiana and Hendrix, Rose and Salvador, Jordi and Herrasti, Alvaro and Girshick, Ross and Kembhavi, Aniruddha and Weihs, Luca},
  booktitle = 	 {Proc. of The 8th Conf. on Robot Learning},
  pages = 	 {408--432},
  year = 	 {2025},
  editor = 	 {},
  volume = 	 {270},
  series = 	 {Proc. of Machine Learning Research},
  month = 	 {06--09 Nov},
  publisher =    {PMLR},
  url = 	 {}
}

@InProceedings{pmlr-v305-sermanet25a,
  title = 	 {Generating Robot Constitutions \& Benchmarks for Semantic Safety},
  author =       {Sermanet, Pierre and Majumdar, Anirudha and Irpan, Alex and Kalashnikov, Dmitry and Sindhwani, Vikas},
  booktitle = 	 {Proc. of The 9th Conf. on Robot Learning},
  pages = 	 {4767--4823},
  year = 	 {2025},
  editor = 	 {},
  volume = 	 {305},
  series = 	 {Proc. of Machine Learning Research},
  month = 	 {27--30 Sep},
  publisher =    {PMLR},
  url = 	 {}
}

@INPROCEEDINGS{SELP,
  author={Wu, Yi and Xiong, Zikang and Hu, Yiran and Iyengar, Shreyash S. and Jiang, Nan and Bera, Aniket and Tan, Lin and Jagannathan, Suresh},
  booktitle={IEEE Int. Conf. on Robotics and Automation (ICRA)}, 
  title={SELP: Generating Safe and Efficient Task Plans for Robot Agents with Large Language Models}, 
  year={2025},
  volume={},
  number={},
  pages={2599-2605},
  doi={10.1109/ICRA55743.2025.11128420}}

@INPROCEEDINGS{10611447,
  author={Yang, Ziyi and Raman, Shreyas S. and Shah, Ankit and Tellex, Stefanie},
  booktitle={IEEE Int. Conf. on Robotics and Automation (ICRA)}, 
  title={Plug in the Safety Chip: Enforcing Constraints for LLM-driven Robot Agents}, 
  year={2024},
  volume={},
  number={},
  pages={14435-14442},
  doi={10.1109/ICRA57147.2024.10611447}}

@InProceedings{10.1007/978-3-319-11164-3_19,
author="Dokhanchi, Adel
and Hoxha, Bardh
and Fainekos, Georgios",
title="On-Line Monitoring for Temporal Logic Robustness",
booktitle="Runtime Verification",
year="2014",
publisher="Springer Int. Publishing",
address="Cham",
pages="231--246",
isbn="978-3-319-11164-3"
}

@INPROCEEDINGS{wang_multi_agent_rl,
  author={Wang, Jiangwei and Yang, Shuo and An, Ziyan and Han, Songyang and Zhang, Zhili and Mangharam, Rahul and Ma, Meiyi and Miao, Fei},
  booktitle={IEEE/RSJ Int. Conf. on Intelligent Robots and Systems (IROS)}, 
  title={Multi-Agent Reinforcement Learning Guided by Signal Temporal Logic Specifications}, 
  year={2025},
  volume={},
  number={},
  pages={6048-6054},
  doi={10.1109/IROS60139.2025.11246629}}

@misc{meng2025tgpotemporalgroundedpolicy,
      title={TGPO: Temporal Grounded Policy Optimization for Signal Temporal Logic Tasks}, 
      author={Yue Meng and Fei Chen and Chuchu Fan},
      year={2025},
      eprint={2510.00225},
      archivePrefix={arXiv},
      primaryClass={cs.RO},
      url={}, 
}

@INPROCEEDINGS{yuksel2026,
  author={Yüksel, Sadık Bera and Buyukkocak, Ali Tevfik and Aksaray, Derya},
  booktitle={American Control Conf. (ACC)}, 
  title={Shielded Reinforcement Learning Under Dynamic Temporal Logic Constraints}, 
  year={2026},
  volume={},
  number={},
  pages={4358-4363},
  doi={}}

@INPROCEEDINGS{LIMP,
  author={Quartey, Benedict and Rosen, Eric and Tellex, Stefanie and Konidaris, George},
  booktitle={IEEE Int. Conf. on Robotics and Automation (ICRA)}, 
  title={Verifiably Following Complex Robot Instructions with Foundation Models}, 
  year={2025},
  volume={},
  number={},
  pages={1-8},
  doi={10.1109/ICRA55743.2025.11127418}}

@InProceedings{pmlr-v270-kim25c,
  title = 	 {OpenVLA: An Open-Source Vision-Language-Action Model},
  author =       {Kim, Moo Jin and Pertsch, Karl and Karamcheti, Siddharth and Xiao, Ted and Balakrishna, Ashwin and Nair, Suraj and Rafailov, Rafael and Foster, Ethan P and Sanketi, Pannag R and Vuong, Quan and Kollar, Thomas and Burchfiel, Benjamin and Tedrake, Russ and Sadigh, Dorsa and Levine, Sergey and Liang, Percy and Finn, Chelsea},
  booktitle = 	 {Proc. of The 8th Conf. on Robot Learning},
  pages = 	 {2679--2713},
  year = 	 {2025},
  editor = 	 {},
  volume = 	 {270},
  series = 	 {Proc. of Machine Learning Research},
  month = 	 {06--09 Nov},
  publisher =    {PMLR},
  url = 	 {}
}

@inproceedings{NEURIPS2025_safeVLA,
 author = {Zhang, Borong and Zhang, Yuhao and Ji, Jiaming and Lei, Yingshan and Dai, Juntao and Chen, Yuanpei and Yang, Yaodong},
 booktitle = {Advances in Neural Information Processing Systems},
 doi = {10.52202/085713-5128},
 editor = {},
 pages = {153335--153373},
 publisher = {},
 title = {SafeVLA: Towards Safety Alignment of Vision-Language-Action Model via Constrained Learning},
 url = {},
 volume = {},
 year = {2025}
}

@INPROCEEDINGS{objaverse,
  author={Deitke, Matt and Schwenk, Dustin and Salvador, Jordi and Weihs, Luca and Michel, Oscar and VanderBilt, Eli and Schmidt, Ludwig and Ehsanit, Kiana and Kembhavi, Aniruddha and Farhadi, Ali},
  booktitle={IEEE/CVF Conf. on Computer Vision and Pattern Recognition (CVPR)}, 
  title={Objaverse: A Universe of Annotated 3D Objects}, 
  year={2023},
  volume={},
  number={},
  pages={13142-13153},
  doi={10.1109/CVPR52729.2023.01263}}

@misc{schulman2017proximalpolicyoptimizationalgorithms,
      title={Proximal Policy Optimization Algorithms}, 
      author={John Schulman and Filip Wolski and Prafulla Dhariwal and Alec Radford and Oleg Klimov},
      year={2017},
      eprint={1707.06347},
      archivePrefix={arXiv},
      primaryClass={cs.LG},
      url={}, 
}

@misc{vanhasselt2015deepreinforcementlearningdouble,
      title={Deep Reinforcement Learning with Double Q-learning}, 
      author={Hado van Hasselt and Arthur Guez and David Silver},
      year={2015},
      eprint={1509.06461},
      archivePrefix={arXiv},
      primaryClass={cs.LG},
      url={}, 
}

\end{document}